%% file: main.tex
\documentclass[final]{alt2023}

\usepackage{caption}
\usepackage{amsfonts}
\usepackage{bbm}
\usepackage{bm}
\usepackage{mathtools}
\usepackage[utf8]{inputenc}
\makeatletter
\newtheorem*{rep@theorem}{\rep@title}
\newcommand{\newreptheorem}[2]{%
\newenvironment{rep#1}[1]{%
 \def\rep@title{#2 \ref{##1}}%
 \begin{rep@theorem}}%
 {\end{rep@theorem}}}
\makeatother
\newtheorem{thm}{Theorem}
\newreptheorem{thm}{Theorem}
\newtheorem{lem}[thm]{Lemma}
\newreptheorem{lem}{Lemma}
\newtheorem{prop}[thm]{Proposition}
\newreptheorem{prop}{Proposition}
\newtheorem{cor}[thm]{Corollary}
\newreptheorem{cor}{Corollary}
\newtheorem{claim}{Claim}
\usepackage{xcolor}

\DeclareMathOperator*{\argmin}{arg\,min}

\newtheorem{innercustomgeneric}{\customgenericname}
\providecommand{\customgenericname}{}
\newcommand{\newcustomtheorem}[2]{%
  \newenvironment{#1}[1]
  {%
   \renewcommand\customgenericname{#2}%
   \renewcommand\theinnercustomgeneric{##1}%
   \innercustomgeneric
  }
  {\endinnercustomgeneric}
}

\newcustomtheorem{customthm}{Theorem}
\newcustomtheorem{customlemma}{Lemma}
\newcustomtheorem{customprop}{Proposition}

\newtheorem{defn}[thm]{Definition}

\usepackage{soul}

\newcommand{\propref}[1]{Proposition~\ref{#1}}

\newcommand{\paren} [1] {\ensuremath{ \left( {#1} \right) }}

\newcommand{\curlybracket}[1]{\ensuremath{\left\{#1\right\}}}

\def\E{\mathbb{E}}
\def\R{\mathbb{R}}

\def\cH{\mathcal H}
\def\reals{\mathbb{R}}
\def\cD{\mathcal D}
\def\Ev{\mathbb{E}}
\def\cO{\mathcal O}
\def\cX{\mathcal X}
\def\cU{\mathcal U}
\def\cM{\mathcal M}

\title[Tolerant Robustness]{Robust Empirical Risk Minimization with Tolerance}
\usepackage{times}
\altauthor{%
 \Name{Robi Bhattacharjee} \Email{rcbhatta@eng.ucsd.edu}\\
 \addr UCSD
 \AND
 \Name{Max Hopkins} \Email{nmhopkin@eng.ucsd.edu}\\
 \addr UCSD%
 \AND
 \Name{Akash Kumar} \Email{akk002@ucsd.edu  }\\
 \addr UCSD%
 \AND
 \Name{Hantao Yu} \Email{hy2751@columbia.edu }\\
 \addr Columbia%
  \AND
 \Name{Kamalika Chaudhuri} \Email{kamalika@eng.ucsd.edu}\\
 \addr UCSD%
}

\begin{document}

\maketitle

\input{0_abstract.tex}


\input{Intro3.tex}
\input{2_related_work.tex}

\input{3_prelim.tex}

\input{4_lower_bound.tex}
\input{5_tolerant_setting.tex}
\input{6_sample_complexity.tex}

%\input{no_substitution.tex}
%\input{algo_description.tex}
%\input{proof.tex}
%\input{no_substitution_revisited.tex}

% Acknowledgments---Will not appear in anonymized version

\bibliography{main.bib}
\newpage
\appendix

%\input{9_table_of_content}
\input{9.1_appendix}

\end{document}

%% file: 0_abstract.tex
%!TEX root = main.tex
\begin{abstract}

% Developing simple, sample-efficient learning algorithms for robust classification is a pressing issue in today’s tech-dominated world, and current theoretical techniques requiring exponential sample complexity and complicated improper learning rules fall far from answering the need. Towards this end, one natural path is to study relaxations of the standard model such as \citet{Urner22}’s \textit{tolerant} learning, a recent notion where the output classifier is compared to the best achievable error over slightly larger perturbation sets. In this work, we show that under weak niceness conditions, achieving simple, sample-efficient robust learning is indeed possible: a natural tolerant variant of robust empirical risk minimization is in fact sufficient for learning over arbitrary perturbation sets of bounded diameter $D$ using only $O\left( \frac{vd\log \frac{dD}{\epsilon\gamma\delta}}{\epsilon^2}\right)$ samples for VC-dimension $v$ hypothesis classes over $\mathbb{R}^d$.

Developing simple, sample-efficient learning algorithms for robust classification is a pressing issue in today’s tech-dominated world, and current theoretical techniques requiring exponential sample complexity and complicated improper learning rules fall far from answering the need. In this work we study the fundamental paradigm of (robust) \textit{empirical risk minimization} (RERM), a simple process in which the learner outputs any hypothesis minimizing its training error. RERM famously fails to robustly learn VC classes \citep{Omar19}, a bound we show extends even to `nice' settings such as (bounded) halfspaces. As such, we study a recent relaxation of the robust model called \textit{tolerant} robust learning \citep{Urner22} where the output classifier is compared to the best achievable error over slightly larger perturbation sets. We show that under geometric niceness conditions, a natural tolerant variant of RERM is indeed sufficient for $\gamma$-tolerant robust learning VC classes over $\mathbb{R}^d$, and requires only $\tilde{O}\left( \frac{VC(H)d\log \frac{D}{\gamma\delta}}{\epsilon^2}\right)$ samples for robustness regions of (maximum) diameter $D$.
\end{abstract}

%% file: Intro3.tex
%!TEX root = main.tex
\section{Introduction}

Adversarially robust classification is a staple of modern machine learning. In the robust setting, along with meeting standard accuracy guarantees, predictions made by a learner at test time must additionally be robust to adversarial perturbations to the input, typically defined by a fixed family $\mathcal{U}=\{U_x\}_{x \in X}$ of possible perturbations. Developing robust algorithms with provable guarantees has been an important research direction in recent years, both for parametric \cite{loh18, attias19, Srebro19, bartlett19, pathak20} and non-parametric \cite{WJC18, YRWC19, Bhattacharjee20, Bhattacharjee21} classifiers, but understanding the performance of even the most basic algorithms in the setting remains open.

In this work, we study one of the simplest, most fundamental algorithmic paradigms in learning, a classical method called \textit{empirical risk minimization} (ERM). In the robust setting, an algorithm is said to be an empirical risk minimizer (RERM) if it always outputs a hypothesis in the class with minimal \textit{robust} risk over its training data. In the standard setting, it is a classical result that any learnable class is learnable (near-optimally) by any ERM. Unfortunately, this is known to fail drastically in the robust setting---\citet{Omar19} showed that there exist finite VC classes, $\mathcal{H}$, where no algorithm outputting hypotheses in $\mathcal{H}$ (called a \textit{proper} learner) can converge towards the optimal classifier, even with arbitrary amounts of training data. Conversely, such classes \textit{are} in fact robustly learnable, but require complicated improper learning rules and a potentially exponential number of samples.

The failure of Robust ERM for general classes raises an interesting question: \textit{are there natural sufficient conditions for the success of RERM?} One obvious answer to this question is the notion of robust VC dimension, a combinatorial parameter promising the success of RERM. However, bounding robust VC is typically difficult, and such results are only known for very specialized examples of classifiers and robustness regions (e.g.\ linear classifiers under fixed-radius balls \citep{Cullina18} and other simple margin structures \citep{pathak20}, or VC-classes under finite perturbation sets \citep{attias19}). To our knowledge there are no corresponding results for more general robustness regions and hypothesis classes beyond these special cases.

% Although there is some recent work indicating this may be possible (for example  gives sample complexity bounds for learning robust linear classifiers), these works are typically restricted to specific examples of classifiers and robustness regions, with the most common example being the case that all robustness regions are ball within some norm of a fixed radius. 
% However, to our knowledge, there are no corresponding results for the general case of arbitrary robustness region.

Given the current failure of combinatorial techniques in this setting, one might instead hope to show RERM works given sufficiently nice \textit{geometric} conditions on the hypothesis class. Sadly, this is not the case. We show that there exist robustness regions for which RERM (indeed any proper algorithm) fails even for settings as simple as (bounded) linear classifiers.
\begin{thm}[Failure of RERM for Linear Classifiers]\label{thm:intro1}
For any $W>0$ and $d>1$, let $\cH_W$ denote the set of linear classifiers with distance at most $W$ from the origin. Then there exists a set of robustness regions $U$ over $\R^d$ such that for any proper learning algorithm $L$ there exists a distribution $\cD$ for which the following hold:
\begin{itemize}
	\item \textbf{$\cD$ is realizable}: There exists $h^* \in \cH_W$ such that $\ell_U(h^*, \cD) = 0$.
	\item \textbf{$L$ has high error}: With probability at least $\frac{1}{7}$ over $S \sim \cD^m$, $\ell_U(L(S), \cD) > \frac{1}{8}$. 
\end{itemize}
\end{thm}

With this in mind, we turn our attention to a different approach: relaxing the notion of robustness itself. We'll consider a recent model of \citet{Urner22} called \textit{tolerant} robust learning. In the tolerant setting, the learner is only required to compete with the best loss over a relaxed family of perturbation sets $\mathcal{U}^\gamma$ for a (potentially arbitrary) tolerance parameter $\gamma >0$. \citet{Urner22} studied this setting in the special case of radius $r$ balls, where the learner competes with robust error against $r(1 + \gamma)$-balls. Under this framework, \citet{Urner22} give an algorithm with PAC-guarantees for VC classes using significantly fewer samples, but their techniques remain improper and only hold for the simplest robustness setting.

In this work, we show that a simple variant of RERM in the tolerant model indeed succeeds under natural geometric conditions on the hypothesis class. In particular, we study a notion of smoothness called \textit{regularity}, which roughly promises that every point in the instance space should be contained in some ball of the same label. This captures many well-studied settings, such as cases where the decision boundaries are compact, differential manifolds in $\mathbb{R}^d$.
\begin{thm}[Tolerant RERM for Regular Classes]\label{thm:upper_bound}
Let $\cH$ be a regular hypothesis class with VC dimension $v$ over $\mathbb{R}^d$, and let $\mathcal{U}$ be any set of robustness regions. Then $TolRERM$ tolerantly PAC-learns $(\cH, \mathcal{U})$ with tolerant sample complexity 
\[
m(\epsilon, \delta, \gamma)  = O\left( \frac{vd\log \frac{d\text{Diam}(U)}{\epsilon\gamma\delta}}{\epsilon^2}\right),
\]
where $\text{Diam}(U)$ denotes the maximum $\ell_2$ diameter across robustness regions $U_x$. 
\end{thm}
Theorem \ref{thm:upper_bound} matches the sample complexity given in \citet{Urner22} up to logarithmic factors and enjoys the additional benefits of applying to more general robustness regions along with its properness and general algorithmic simplicity. For completeness, we also analyze our algorithm's performance over non-regular classifiers in Appendix \ref{sec:proof_extension}, and show that it has a similar performance albeit at the cost of replacing the VC-dimension with $v_{\text{ball}}$, the robust VC dimension of $\cH$ over balls of a fixed radius. Thus, for non-regular hypothesis classes, our algorithm gives a reduction from arbitrary robustness regions to the case where they are all balls of a fixed radius.

Finally it's worth noting that while \citet{Urner22} only requires sampling access to the perturbation sets, stronger access such as an empirical risk minimizer is inevitable in the general setting where $\mathcal{U}$ is unknown. We show that there exists hypothesis classes where $\Omega((\frac{D}{\gamma})^d)$ queries to a sampling oracle are required for robust learning with tolerance if no other interaction with $U_x$ is permitted.

While Theorem \ref{thm:upper_bound} gives a natural sufficient condition for the success of RERM in relaxed settings, many questions in this direction remain wide open. It would be interesting to identify a necessary condition for the success of RERM, both in the tolerant and original robust models. Furthermore, it should be noted that while we prove RERM fails to learn nice classes in the latter, the perturbation family we use to achieve this is highly combinatorial. As such, there is still hope that RERM may be sufficient in the traditional setting under \textit{joint} niceness conditions on $\mathcal{H}$ and $\mathcal{U}$, though the close interplay between the two families seems to make identifying such a condition difficult, if it is indeed possible at all.

%% file: 2_related_work.tex
%!TEX root = main.tex
\section{Related Work}

Much of the work on adversarial robustness \citep{Carlini17, Liu17, Papernot17, Papernot16, Szegedy14, Hein17, Katz17, Wu16,Steinhardt18, Sinha18} is done in the context of neural networks.

On the theoretical side, there has been a recent focus on developing algorithms with guarantees in convergence towards an optimal classifier. On the parametric side, several works \citep{loh18, attias19, Srebro19, bartlett19, pathak20, Cullina18} have focused on distribution agnostic bounds on the amount of data required to converge towards the optimal classifier in a given hypothesis class. For example, \citet{Srebro19} showed through an example that the VC dimension of robust learning may be much larger than standard or accurate learning indicating that the sample complexity bounds may be higher. There has also been some work considering the computation complexity required for robust learning such as \citet{Kane20}.

Aside from \citet{Urner22}, there are several works which also consider variations on robust learning with tolerance. \citet{YRWC19} and \citet{Bhattacharjee20} show that certain non-parametric algorithms exhibit a type of tolerant behavior when robustness regions are constrained to be balls of radius $r$. \citet{Omar22} considers robustness in the \textit{transductive learning setting}. Their work employs a similar idea to \citet{Urner22} in that they consider expanded perturbation sets when giving their formal guarantees. However, their expansions are not based on tolerance $\gamma > 0$.

Finally, \cite{Awasthi21}, introduces a notion of pseudo-robustness which precisely matches our definition of a regular classifier (Definition \ref{defn:regular}). Their work focuses on using this notion to define a robust analog to the Bayes optimal classifier. By contrast, our work focuses on learning a robust hypothesis class that satisfies this condition.

%% file: 3_prelim.tex
%!TEX root = main.tex
\section{Preliminaries}
Let $\cH$ be a family of binary classifiers $\{h: \mathbb{R}^d \to \{\pm 1\}\}$, and $U = \{U_x \subseteq \R^d: x \in \reals^d\}$ any set of robustness regions. 
% We make no assumptions about $U$ except for each $U_x$ having $\ell_2$ diameter at most $D$. 
We define the robust loss function with respect to $U$ as follows.
\begin{defn}
Let $h \in \cH$ be a classifier and $(x,y) \in \reals^d \times \{\pm 1\}$ be a labeled point. Then the \textbf{robust loss} of $h$ over $(x, y)$, denoted $\ell_U(h, (x,y))$, is defined as 
\begin{align*}
  \ell_{U}(h, (x,y)) = \begin{cases} 1 & \exists x' \in U_x\text{ such that }h(x') \neq y \\0 & \text{otherwise.} \end{cases}.  
\end{align*}
%$$$$ 
That is, $h$ achieves a loss of $0$ only if it labels all points in $U_x$ as $y$. 
\end{defn}

For a distribution, $\cD$ over $\reals^d \times \{\pm 1\}$, we let $ \ell_U(h, \cD)$ denote the expected loss $h$ pays over a labeled point drawn from $\cD$. That is, $\ell_U(h, \cD) = \Ev_{(x,y) \sim \cD}[\ell_U(h, (x,y))]$. 

Similarly, for a set of $n$ labeled points, $S$, we let $\ell_U(h, S)$ denote the average robust loss $h$ pays over $S$. that is, $\ell_U(h, S) = \frac{1}{n} \sum_{i=1}^n \ell_U(h, (x_i, y_i))$. 

We will also use $||x - x'||$ to denote the $\ell_2$ distance between $x$ and $x'$, and $B(x, r)$ to denote the (closed) $\ell_2$ ball centered at $x$ with radius $r$.

\subsection{Robust PAC-learning}

We now review a natural generalization of PAC learning to the robust setting called robust PAC-learning \citep{Omar19}. 

\begin{defn}\label{defn:rob_pac}
Let $\cH$ be a hypothesis class and $U$ be a set of robustness regions. A learner $L$ \textbf{robustly PAC-learns} $(\cH, U)$ if for every  $\epsilon, \delta > 0$, there exists $m(\epsilon, \delta)$ such that for all $n \geq m(\epsilon, \delta)$, for all data distributions, $\cD$, with probability $1-\delta$ over $S \sim \cD^n$, $$\ell_{U}(\hat{h}, \cD) \leq \min_{h \in \cH} \ell_{U}(h, \cD) + \epsilon,$$ where $\hat{h} = L(S)$ denotes the classifier in $\cH$ outputted by $L$ from training sample $S$. $m(\epsilon, \delta)$ is said to be the \textbf{sample complexity} of $L$ with respect to $(\cH, U)$. 
\end{defn}

Algorithms that are able to robustly PAC-learn a pair $(\cH, U)$ are the natural robust analogs of standard learning algorithms, and thus an important question is understanding how the sample complexities, $m(\epsilon, \delta)$, for doing so are bounded.

%% file: 4_lower_bound.tex
%!TEX root = main.tex
\section{Robust Empirical Risk Minimization on Linear Classifiers}

\looseness-1\citet{Omar19} showed that there exist hypothesis classes $\cH$ with bounded VC dimension, and robustness regions $U$, such that proper robust PAC-learning is not possible, meaning no matter how much data one is allowed, there always exists a distribution where the learner will suffer high robust loss.

However, for many practical examples, this does not appear to be the case -- for example, \cite{Cullina18} showed that when $\cH$ is the set of all linear classifiers and $U$ is the set of robustness regions with $U_x = B(x, r)$, the sample complexity of robustly learning with RERM is at most $m(\epsilon, \delta) = \tilde{O}\left(\frac{d}{\epsilon^2}\right)$, matching the standard complexity for linear classification. 

Motivated by recent interest in more general robustness regions than balls of a fixed radius, we consider the case where $\cH$ is a natural hypothesis class, but $U$ is a potentially arbitrary robustness region. That is, we ask the following question: are there examples of natural hypothesis classes for which there exist robustness regions leading to arbitrary high sample complexities?

Unfortunately, the answer turns out to be yes. To show this, we begin by defining the natural hypothesis class of \textit{bounded} linear classifiers.
\begin{defn}\label{defn:bounded_linear}
A $W$-bounded linear classifier, $f: \R^d \to \R^d$, is a linear classifier $h$ whose decision boundary has distance at most $W$ from the origin. That is, there exist $w \in \R^d$ and , $b\in \R$ with $\frac{|b|}{||w||} \leq W$ such that $$h(x) = \begin{cases}1 & \langle w, x \rangle + b \geq 0 \\ -1 & otherwise \end{cases}.$$ We let $\cH_W$ denote the class of all $W$-bounded linear classifiers
\end{defn}
The boundedness condition, $W$, can be thought of as a regularization term which is common during any kind of practical optimization.

We now show that there exist robustness regions, $U$, for which $(\cH_W, U)$ is not robustly PAC-learnable, even in the realizable setting. For convenience, we restate Theorem \ref{thm:lower_bound} from the introduction.

% \begin{theorem}\label{thm:lower_bound}
% Let $W > 0$ be arbitrary and let $m > 0$ be any integer. Then there exists a set of robustness regions $U$ over $\R^d$ such that for any learning algorithm $L$, there exists a distribution $\cD$ for which the following hold:
% \begin{itemize}
% 	\item There exists $h^* \in \cH_W$ such that $\ell_U(h^*, \cD) = 0$.
% 	\item With probability at least $\frac{1}{7}$ over $S \sim \cD^m$, $\ell_U(L(S), \cD) > \frac{1}{8}$. 
% \end{itemize}
% \end{theorem}
\begin{theorem}\label{thm:lower_bound}
For any $W>0$, $d>1$, and $m>1$, there exists a set of robustness regions $U$ over $\R^d$ such that for any learning algorithm $L$ there exists a distribution $\cD$ for which the following hold:
\begin{itemize}
	\item \textbf{$\cD$ is realizable}: There exists $h^* \in \cH_W$ such that $\ell_U(h^*, \cD) = 0$.
	\item \textbf{$L$ has high error}: With probability at least $\frac{1}{7}$ over $S \sim \cD^m$, $\ell_U(L(S), \cD) > \frac{1}{8}$. 
\end{itemize}
\end{theorem}

% Theorem \ref{thm:lower_bound} immediately implies that there exist $U$ for which the sample complexity, $m(\epsilon, \delta)$ is arbitrarily high for any learner.
Theorem \ref{thm:lower_bound} consequently shows that the observations made in \cite{Omar19} hold even over practical hypothesis classes such as (bounded) linear classifiers.

To prove Theorem \ref{thm:lower_bound}, we begin with the following critical lemma.
\begin{lem}\label{lem:finding_shatter}
For every $M \in \mathbb{N}$ there exists a family of $M$ subsets of $\R^d$
\[
Z^{(M)} \coloneqq \left\{Z^{(M)}_1, Z^{(M)}_2, \dots, Z^{(M)}_M\right\}
\]
satisfying the following conditions:
\begin{itemize}
	\item There exists $1 \leq i \leq M$ and $z \in Z_i^{(M)}$ with $h(z) = 1$. 
	\item For every $1 \leq i \leq M$, there exists $h_i \in \cH_W$ such that $h_i(z) = -1$ for all $z \in \cup_{j \neq i} Z^{(M)}_j$.
	\item For any distinct natural numbers $M$ and $M’$, the sets $\cup_{i= 1}^M Z_i^M$ and $\cup_{i=1}^{M’} Z_i^{M’}$ are disjoint. Thus, there is no point that is contained in subsets from both $Z^M$ And $Z^{M’}$.
\end{itemize}
\end{lem}

\begin{proof}
Let $\{\beta_i\}_{i \in \mathbb{N}} > 0$ be a strictly decreasing sequence of sufficiently small real numbers (that we will specify later). For notational simplicity, fix an $M \in \mathbb{N}$ and write $\beta=\beta_M$ and $W' = (1 + \beta)W$. For any $r > 0$, let $S_r^{d-1}$ denote the $(d-1)$-sphere centered at the origin of radius $r$.

Observe that for any $x \in S_W^{d-1}$, there exists a unique classifier $h \in \cH_W$ whose decision boundary is tangent to $S_W^{d-1}$ at $x$ so that $h(x) = 1$. We denote this classifier as $h_x$. It follows that the set of all points on $S_{W'}^{d-1}$ that $h_x$ classifies as $1$ can be easily characterized in terms of $x$. In particular, by the definition of $h_x$, it follows from geometry that
\begin{equation}\label{eqn:lol_i_literally_used_law_of_cosines}
\curlybracket{z: h_x(z) = 1, z \in S_{W'}^{d-1}} = \curlybracket{z: ||z - (1+\beta)x|| \leq W\sqrt{2\beta(\beta + 1)}, z \in S_{W'}^{d-1}}.
\end{equation}

Let $r_\beta = 2W\sqrt{2\beta(\beta + 1)}$, and let $z_1, z_2, \dots, z_{M_{\beta}}$ denote a a greedy $r_\beta$ cover of $S_{W'}^{d-1}$, meaning that points are successively selected from $S_{W'}^{d-1}$ until no point with distance strictly greater than $r_\beta$ from all other points can be selected. Finally, define $Z_i=Z_i^{(M)}$ as the set of elements in $S_{W'}^{d-1}$ with nearest neighbor $z_i$ (ties broken arbitrarily).

We claim that this construction suffices for $M_\beta \geq M$. First, observe that $\lim_{\beta \to 0} r_\beta = 0$, which means that for sufficiently small $\beta$ that $M_\beta$ will be arbitrarily large (thus satisfying $M_\beta \geq M$). So select any $\beta$ for which this hold, and merge enough regions so that we are left with exactly $M$ regions (i.e. set $Z_{M} = \cup_{i = M}^{M_\beta} Z_i$). Note that we can always choose $0<\beta< \beta_{M-1}$ since the naturals can be embedded into any interval. We now verify the two stipulations of Lemma \ref{lem:finding_shatter}. 

The first stipulation clearly holds since $\{Z_i\}_{i=1}^M$ partition $S_{W'}^{d-1}$ and every halfspace $h \in \cH_W$ intersects the latter by construction.

For the second stipulation, observe that for any $i$, the ball centered at $z_i$ of radius $\frac{r_\beta}{2}$, $B\left(z_i, \frac{r_\beta}{2}\right)$, does not intersect $Z_j$ for any $i \neq j$. This is because such an intersection would imply by the triangle inequality that $||z_i - z_j|| \leq r_\beta$, which is a contradiction. This observation allows us to find a classifier, $h_i$, as desired --  we set $h_i$ to be the previously defined classifier, $h_{\frac{z_i}{1 + \beta}}$. Equation \ref{eqn:lol_i_literally_used_law_of_cosines} implies that the only points in $S_{W'}^{d-1}$ that it will classify as $1$ are precisely the points in $B\left(z_i, \frac{r_\beta}{2}\right) \cap S_{W'}^{d-1}$. Since this is a subset of $Z_i$, the second stipulation is met, as desired.

Finally, it is left to observe that over each choice of $M$ these $Z^{(M)}$ are mutually disjoint. This is true so long as the choices of $\beta$ themselves are disjoint, since $Z^{(M)}$ lies in the sphere of radius $W(1+\beta_M)$. As noted previously it is easy to see $\{\beta_M\}$ can be chosen in this manner in an inductive fashion.
\end{proof}

We are now sketching a proof for Theorem \ref{thm:lower_bound}, with the full proof deferred Appendix \ref{sec:lower_bound_proof}.

\paragraph{Proof Sketch: (Theorem \ref{thm:lower_bound})} 

Our goal is to show that for any $m \in \mathbb{N}$, any learner on $m$ samples must fail with constant probability. Fix any $m$. The main idea will be to construct a set of robustness regions, $U_{x_1}, U_{x_2}, \dots, U_{x_{3m}}$ such that any classifier in $\cH_W$ will lack robustness on at least $m$ of them.  T

Toward this end, set $M = \binom{3m}{m}$, and let $Z^{(M)}_1, Z^{(M)}_2, \dots, Z^{(M)}_M$ be subsets of $\R^d$ as described by Lemma \ref{lem:finding_shatter} (we will drop the superscript in what follows). Let $\mathcal{M}$ denote the set of all subsets of $\{1, \dots, 3m\}$ with exactly $m$ elements. Associate with each $Z_i$ a unique element of $\mathcal{M}$, thus allowing us to rename our subsets as $\{Z_T: T \in \mathcal{M}\}.$ We now define $$U_{x_i} = \cup_{T: i \in T} Z_T,$$ where $x_i$ is an arbitrary point inside $U_{x_i}$.

Lemma \ref{lem:finding_shatter} that if all $x_i$ are given a label of $-1$, then any $h \in \cH_W$ will label some (for some set $T$) some $z \in Z_T$ as $+1$, thus causing it to lack robustness on \textit{all} $i \in T$. Conversely, we see that for any $T$, there is a classifier $h_T \in \cH_W$ that is accurate and robust at all $x_i$ with $i \notin T$.

With these observations, we are now prepared to show that for any learner $L$, there exists a distribution $D$ for which $L$ has large expected robust loss. To do this, we use a standard lower bound technique found in \cite{ml_book} that was adapted to the robust setting in \cite{Omar19}. The idea will be to pick $D$ to be the uniform distribution over a random subset of $2m$ points in $\{x_1, \dots, x_{3m}\}$. We will then argue that because $L$ only has access to $m$ points from $D$, it won't be able to distinguish which subset $D$ corresponds to, and this will lead to a large expected loss. $\square$

As demonstrated in Lemma \ref{lem:finding_shatter}, the robustness regions $U$ used in our lower bound are combinatorial in nature and unlikely to represent any practical kinds of robustness regions. Nevertheless, our lower bound does show that naturality assumptions on the hypothesis class alone are \textit{not} sufficient for ensuring robust PAC-learnability.

A natural next step would be to fully characterizes pairs $(\cH, U)$ for which proper robust PAC-learnability is possible, but we leave this as a direction for future work. We instead turn towards relaxing the requirements of the robust PAC-learning model in order to find algorithms that are able to succeed in the case that $\cH$ is natural but $U$ is arbitrary. 

%% file: 5_tolerant_setting.tex
%!TEX root = main.tex

\section{Tolerant PAC learning}

Theorem \ref{thm:lower_bound} implies that for complex robustness region, robust PAC-learning (Definition \ref{defn:rob_pac}) is not possible, even when $\cH$ is a very simple hypothesis class. Thus, robust learning will require other ideas.

One such idea is Tolerant PAC-learning, introduced in \citet{Urner22}. Here, the idea is to relax the goal of robust PAC-learning by introducing a tolerance parameter $\gamma$ representing the amount of ``slack" the learner gets with respect to the robustness regions $U$. We now expand their definition to arbitrary robustness regions by introducing \textit{perturbed regions}, $U^\gamma$, which are defined as follows.  

\begin{defn}
Let $U$ be a set of robustness regions and $\gamma > 0$ be a distance. For any point $x \in \reals^d$, define $U_x^\gamma$ as the set of all points with distance at most $\gamma$ from $U_x$. That is, $$U_x^\gamma = \{x': ||x' - U_x|| \leq \gamma\}.$$ Finally, we let $U^\gamma = \{U_x^{\gamma}: x \in \reals^d\}$ denote the set of \textbf{$\gamma$-perturbed regions} of $U$. 
\end{defn}

Tolerant PAC-learning is then defined as follows
\begin{defn}\label{defn:tol_pac}
Let $\cH$ be a hypothesis class and $U$ a set of robustness regions. A learner $L$ \textbf{tolerantly PAC-learns} $(\cH, U)$ if for every  $\epsilon, \delta, \gamma > 0$, there exists $m(\epsilon, \delta, \gamma)$ such that for all $n \geq m(\epsilon, \delta, \gamma)$, for all data distributions, $\cD$, with probability $1-\delta$ over $S \sim \cD^n$, $$\ell_U(\hat{h}, \cD) \leq \min_{h \in \cH} \ell_{U^\gamma}(h, \cD) + \epsilon,$$ where $\hat{h} = L(S)$ denotes the classifier outputted by $L$ from training sample $S$. As before, we let $m(\epsilon, \delta, \gamma)$ denote the \textbf{tolerant sample complexity} of $L$ with respect to $(\cH, U)$. 
\end{defn}

\subsection{Tolerant RERM oracles}

Because our robustness regions, $U_x$, are arbitrary subsets of $\R^d$, any learning algorithm will require some sort of access to $U$. We describe this access through an oracle for $U$. 

\citet{Urner22} employs a \textit{sampling oracle} for $U$ which allows the learner to sample points at uniform from the set $U_x$ for any point $x$. In their setting, $U_x$ is constrained to be a closed ball of known radius centered at $x$, and consequently the sampling oracle selects points from the uniform distribution over the ball. We say that a robust learner is in the \textit{sampling model} if its only way of interacting with the regions $U_x$ is through a sampling oracle.

In our setting, where $U_x$ can be an arbitrary regions, sampling oracles pose a significant challenge -- there exists choices of $U$ for which tolerant PAC learning requires an exponential number of queries to the sampling oracle. We state this as a proposition with the proof deferred to Appendix \ref{app: lower bound}.
\begin{prop}\label{prop: lowerbound}
For any $D>10\gamma > 0$, there exists a hypothesis class $\mathcal{H}$ and a set of robustness regions, $U$ such that the following holds. There exist constants $\epsilon$ and $\delta$ such that for any $n > 0$, any learner $L$ on $n$ samples that achieves 
$$\ell_U(L(S), \cD) \leq \min_{h \in \cH} \ell_{U^\gamma}(h, \cD) + \epsilon$$
with probability at least $1-\delta$ must make at least $\Omega\left(\left(\frac{D}{\gamma}\right)^d\right)$ calls to the sampling oracle on some valid data distribution $\cD$, . 
\end{prop}

To circumvent this issue, we turn our attention to a different natural oracle first proposed in \citet{Omar19} that is based on Robust Empirical Risk Minimization (RERM). An RERM oracle, $\cO_{U, \cH}(S)$, is a function that returns the classifier $h \in \cH$ with minimal robust empirical risk over $S$. That is, $$\cO_{U, \cH}(S) = \argmin_{h \in \cH} \ell_U(h, S).$$ In our work, we will assume access to a mild strengthening of this oracle that allows empirical risk minimization over any perturbed robustness region, $U^r$.

\begin{defn}\label{defn:oracle}
\textbf{A tolerant RERM-oracle} for robustness regions $U$ and hypothesis class $\cH$ is a function $\cO_{U, \cH}(S, r)$ that maps any set of labeled points $S$ and any distance $r > 0$ to the classifier with minimal empirical risk over $S$ with respect to $U^r$. That is, $$\cO_{U, \cH}(S, r) = \argmin_{h \in H} \ell_{U^r}(h, S).$$
\end{defn}

Observe that in the case that $U$ consists of balls of radius $r$, a tolerant oracle merely implies we can also minimize empirical risk for balls of larger radii. 

%% file: 6_sample_complexity.tex
%!TEX root = main.tex
\section{Tolerant PAC learning for Regular Hypothesis Classes}

Before presenting our algorithm, we first present a key assumption on our hypothesis class, $\cH$, that we refer to as \textit{regularity.} 

\subsection{Regular hypothesis classes}

\begin{defn}\label{defn:regular}
We say that a hypothesis class, $\cH$ is \textbf{$\alpha$-regular} for $\alpha > 0$ if for all $h \in \cH$ and for all $x \in \reals^d$, there exists a closed ball $B$ of radius $\alpha$ containing $x$ such that $h(x') = h(x)$ for all $x' \in B$. We also say that $\cH$ is \textbf{regular} if it is $\alpha$-regular for some $\alpha > 0$. 
\end{defn}

This notion was previously introduced in \cite{Awasthi21} as \textit{pseudo-robustness}.

One important type of classifiers satisfying this condition are hypothesis classes with relatively smooth manifolds as decision boundaries. In particular, the parameter $\alpha$ can be tied to the smoothness measure of a manifold known as its \textit{reach}.

\begin{defn}
Let $M$ be a closed manifold embedded in $\reals^d$. The \textbf{reach} of $M$ is the largest $\alpha > 0$ such that for all $x \in \reals^d$, if $||x - M|| \leq \alpha$, then $x$ has a unique nearest neighbor in $M$.
\end{defn}

This parameter directly translates to regularity.

\begin{prop}\label{prop:reach}
Let $h$ be a classifier with decision boundary $M$. Suppose that $M$ is a closed $(d-1)$-dimensional submanifold over $\R^d$ with reach $\alpha$. Then $h$ is $\alpha/2$-regular. 
\end{prop}

\begin{proof}
Let  $h \in \cH$ be a classifier with decision boundary $M$. Let $x$ be an arbitrary point with $h(x) = y$. We desire to exhibit a ball $B$ of radius $\alpha/2$ containing $x$ for which $h$ is uniformly $y$. 

Let $\rho: \reals^d \to \reals_{\geq 0}$ be the distance function $\rho(x) = ||x - M||$. It is well known that this function is everywhere continuous and has a continuous derivative over $\{x: 0 < \rho(x) < \alpha\}.$

If $\rho(x) > \alpha/2$, then we can simply take $B= B(x, \alpha/2)$ as all points here must be classified as $y$ by the definition of a decision boundary. Thus, assume $\rho(x) \leq \alpha/2$. 

Let $V$ be the gradient vector field of $\rho$ defined over $\{x: \rho(x) < \alpha\}$. Since all points in this region have a unique nearest neighbor in $M$, it becomes clear that the gradient has magnitude $1$ for all such points, and the direction is precisely opposite the straight line path from the point's nearest neighbor in $M$. 

\looseness-1Since $V$ is continuous, (and Lipshitz over a bounded region), there exists a unique curve $\tau$ starting at $x$ of length $\frac{\alpha}{2}$ that is always tangent to $V$. It follows that the endpoint of this path, $x'$ must satisfy $\rho(x') = \frac{\alpha}{2} + \rho(x) > \frac{\alpha}{2}$ and $||x - x'|| \leq \frac{\alpha}{2}$. This means that $B = B(x', \frac{\alpha}{2})$ suffices, as desired. 
\end{proof}

\vspace{-5mm}
\subsection{Our Algorithm} 

\looseness-1 We now give a tolerant PAC learning algorithm called $TolRERM$ (Algorithm \ref{alg:estimate}) which assumes access to a tolerant RERM oracle (Definition \ref{defn:oracle}). $TolRERM$ is essentially robust empirically risk minimization with a slight modification: rather than using the original robustness regions, $U$, we use the perturbed regions, $U^r$ where $0 < r < \gamma$ is chosen at random. $TolRERM$'s performance is given by Theorem \ref{thm:upper_bound}, which is restated here for convenience. 
\begin{theorem}
\looseness-1Let $\cH$ be a regular hypothesis class with VC dimension $v$, and let $U$ be any set of robustness regions. Then $TolRERM$ tolerantly PAC-learns $(\cH, U)$ with tolerant sample complexity $m(\epsilon, \delta, \gamma)  = O\left( \frac{vd\log \frac{dD}{\epsilon\gamma\delta}}{\epsilon^2}\right)$, where $D$ denotes the maximum $\ell_2$ diameter of any robustness region, $U_x$. 
\end{theorem}
\vspace{-2mm}
\begin{algorithm}
   \caption{$TolRERM(\cD, \epsilon, \delta, \gamma, n)$}
   \label{alg:estimate}

   Sample $r \sim [\frac{\epsilon\delta\gamma}{7}, \gamma]$ at uniform\;
   
   Sample $S \sim \cD^n$\;
    
   Output $\hat{h} = \cO_{U, \cH}(S, r)$\;

\end{algorithm}
\vspace{-2mm}
Since the set of bounded linear classifiers, $\cH_W$ (Definition \ref{defn:bounded_linear}) is clearly regular and has VC dimension $O(d)$, Theorem \ref{thm:upper_bound} immediately implies the following corollary. 

\begin{cor}
For any set of robustness regions, $U$, $TolRERM$ tolerantly PAC-learns $(\cH_W, U)$ with tolerant sample complexity $m(\epsilon, \delta, \gamma) = O\left( \frac{d^2\log \frac{dD}{\epsilon\gamma\delta}}{\epsilon^2}\right)$, where $D$ denotes the maximum $\ell_2$ diameter of any robustness region, $U_x$.  
\end{cor}

Observe that $TolRERM$ matches the known sample complexities for linear classifiers found in \cite{Omar19} and \cite{Urner22}. However, it enjoys the advantage of being simpler (as it is essentially an empirical risk minimization algorithm) and a \textit{proper} learning algorithm (as it outputs a linear classifier).

\paragraph{Beyond regular hypothesis classes:} It turns out that Algorithm \ref{alg:estimate} has bounded sample complexity for \textit{any} hypothesis class with finite robust VC-dimension for balls (see Appendix \ref{sec:proof_extension} for a full description).
% However, this comes at the expense of using the adversarial vc dimension, $v_{ball}$, which is the vc dimension of the loss function when the robustness regions are restricted to be balls of fixed radii centered at each point (see Appendix \ref{sec:proof_extension} for a full description). 
Thus, Algorithm \ref{alg:estimate} can alternatively be thought of as a reduction from the sample complexity for learning robust classifiers over arbitrary robustness regions to the sample complexity for balls of fixed radii. This is expressed in the following result (proved in Appendix \ref{sec:proof_extension}). 

\begin{theorem}\label{thm:upper_bound_general}
Let $\cH$ be any hypothesis class with maximal adversarial VC dimension $v_{ball}$, and let $U$ be any set of robustness regions. Then $TolRERM$ tolerantly PAC-learns $(\cH, U)$ with tolerant sample complexity $m(\epsilon, \delta, \gamma)  = O\left( \frac{v_{ball}d\log \frac{dD}{\epsilon\gamma\delta}}{\epsilon^2}\right),$ where $D$ denotes the maximum $\ell_2$ diameter of any robustness region, $U_x$. 
\end{theorem}

\subsection{Proof of Theorem \ref{thm:upper_bound}}

We begin by showing that randomly choosing $r$ allows the optimal empirical loss $U^r$ to change relatively smoothly with respect to $r.$
\begin{lem}\label{lem:r_works}
For $r \in [0, \gamma]$, let $OPT_S^r = \min_{h \in H} \ell_{U^r}(h, S)$. Then with probability at least $1 - \frac{\delta}{2}$ over $r \sim [\frac{\epsilon\delta\gamma}{7}, \gamma]$, $OPT_S^r \leq OPT_S^{r -\frac{\epsilon\delta\gamma}{7}} + \frac{\epsilon}{3}.$
\end{lem}
\textit{Proof. }Let $\alpha = \frac{\epsilon\delta\gamma}{7}.$ Our goal is to show that $OPT_S^r - OPT_S^{r-\alpha}$ is likely to be small. Our strategy is to bound the expected value of $OPT_S^r - OPT_S^{r-\alpha}$ and then apply Markov's inequality. As a technical note, the function $r \mapsto OPT_S^r$ is monotonic and bounded, and consequently measurable, which ensures that our expectations are well defined. To this end, we have,
\begin{equation*}
\begin{split}
\Ev[OPT_S^r &- OPT_S^{r-\alpha}] = \Ev[OPT_S^r] - \Ev[OPT_S^{r-\alpha}] \\
&= \frac{1}{\gamma - \alpha}\left(\int_\alpha^\gamma OPT_S^rdr - \int_\alpha^\gamma OPT_S^{r-\alpha}dr \right) \\
&= \frac{1}{\gamma - \alpha}\left(\int_\alpha^\gamma OPT_S^rdr - \int_0^{\gamma-\alpha} OPT_S^rdr \right) \\
&= \frac{1}{\gamma - \alpha}\left(\int_{\gamma-\alpha}^\gamma OPT_S^rdr - \int_0^{\alpha} OPT_S^rdr \right) \\
&\leq \frac{\alpha}{\gamma - \alpha} \\
&= \frac{\delta\epsilon\gamma}{7\gamma - \delta\epsilon\gamma} \\
&\leq \frac{\delta\epsilon}{6},
\end{split}
\end{equation*}
since $\epsilon, \delta \leq 1$. Applying Markov's inequality, with probability at least $1 - \frac{\delta}{2}$, $OPT_S^r - OPT_S^{r-\alpha} \leq \frac{\epsilon}{3}$. $\square$.

Next, we construct a set of robustness regions $V^r$ that have similar robust loss to $U^r$ and are also finite.

\begin{lem}\label{lem:v_construct}
Suppose that $\cH$ is $\gamma$-regular. For all $r \in [\frac{\epsilon\delta\gamma}{7}, \gamma]$, there exists a set of robustness regions $V^r = \{V_x^r: x \in \reals^d\}$ satisfying the following two properties.  
\begin{enumerate}
	\item $|V_x^r| = O\left(\left(\frac{D}{\epsilon\delta\gamma}\right)^d\right)$, where $D$ denotes the maximum diameter of $U_x$. 
	\item Let $\alpha = \frac{\epsilon\delta\gamma}{7}$. For all labeled points $(x,y)$ and for all classifiers $h \in \cH$, $$\ell_{U^{r - \alpha}}(h, (x,y)) \leq \ell_{V^r}(h, (x,y)) \leq \ell_{U^r}(h, (x,y)).$$
\end{enumerate}
\end{lem}

\begin{proof}
For any $x \in \reals^d$, we will show how to construct $V_x$ so that it satisfies the two conditions above.

Observe that $U_x^r$ is closed and bounded as it is a union of closed balls of radius $r$. Since each $U_x$ has diameter at most $D$, this means that $U_x^r$ is compact. Thus, there exists a finite set of balls of radius $\alpha/2$ that cover $U_x^r$. Note that these balls are \textit{not} necessarily contained within $U_x^r$ -- only that $U_x^r$ is a subset of their union. Let $C_x$ denote the set of all centers of the smallest such cover. We claims that $V_x = C_x \cap U_x^r$ suffices.

First, $|C_x| \leq O\left((\frac{D}{\alpha})^d\right)$ because any ball of diameter $D$ can be covered by $O\left((\frac{D}{\alpha})^d\right)$ balls of radius $\alpha/2$, and $U_x^r$ is a subset of a ball of diameter $D+2r$. This implies that the first condition holds.

Second, pick any labeled point $(x,y)$ and any classifier $h \in \cH$. If $\ell_{V^r}(h, (x,y)) = 1$, then we immediately have $\ell_{U^r}(h, (x,y) = 1$ since $V^r \subseteq U^r$. This implies that $\ell_{V^r}(h, (x,y)) \leq \ell_{U^r}(h, (x,y)$ giving the second half of the second condition.   

If $\ell_{U^{r-\alpha}}(h, (x,y)) = 1$, then there exists $x' \in U_x^{r-\alpha}$ such that $h(x') \neq y$. It follows that since $h$ is $\gamma$-regular, $h$ must also be $\alpha$-regular (as $\alpha < \gamma$). This means that there exists a ball $B$ of radius $\alpha/2$ containing $x'$ such that $h$ does not output $y$ for any point in $B$. 

By the triangle inequality, $B \subseteq U_x^r$, and since $C_x$ covers $U_x^r$, it follows that there exists $x^* \in C_x \cap B$. By definition, this also means $x^* \in V_x^r$. However, by the definition of $B$, we must have $h(x^*) \neq y$, and this means that $\ell_{V_x^r}(h, (x,y)) = 1$. Since $(x,y)$ was arbitrary, this proves the second half of the second condition. 
\end{proof}

We are now prepared to prove Theorem \ref{thm:upper_bound}.

\begin{proof}
(\textbf{Theorem \ref{thm:upper_bound}}) Let $\alpha = \frac{\epsilon\delta\gamma}{7}$. For all $s > 0$, let $h^s \in \cH$ denote any fixed choice of classifier with minimal empirical loss with respect to $U^s$. That is, $$h^s = \argmin_{h \in \cH} \ell_{U^s}(h, S).$$ It suffices to show that with probability at least $1-\delta$ over $S \sim \mathcal{D}^n$ and $r \sim \left[\alpha, \gamma\right]$, $$\ell_U(h^r, \mathcal{D}) \leq \min_{h \in \mathcal{H}}\ell_{U^\gamma}(h, \mathcal{D}) + \epsilon.$$ Let $h^* = \argmin_{h \in \mathcal{H}} \ell_{U^\gamma}(h, \mathcal{D})$, and let $V^r$ be the set of robustness regions defined in Lemma \ref{lem:v_construct}. Then by Lemma \ref{lem:v_construct} and the fact that $r \leq \gamma$, 
\begin{equation}\label{eqn:getting_to_V}
\ell_{U^\gamma}(h^*, \mathcal{D}) \geq \ell_{U^r}(h^*, \mathcal{D}) \geq \ell_{V^r}(h^*, \mathcal{D}).
\end{equation}
Next, since $|V_x| = O\left(\left(\frac{D}{\epsilon\delta\gamma}\right)^d\right)$, Proposition \ref{prop:finite-RVC} (proved in the Appendix \ref{app: robust vc bound}) implies that the Robust VC dimension of $\cH$ with respect to $V_x$ is at most $O\left(vd\log \frac{Dv}{\epsilon\delta\gamma} \right)$, where $v$ denotes the VC dimension of $\cH$.

Because $S$ is independent from $r$, there exists a constant $C$ such that if $n \geq C\frac{vd\log \frac{Dv}{\epsilon\delta\gamma} +\log\frac{1}{\delta}}{\epsilon^2}$, then classical connections with uniform convergence \cite{vapnik1974theory} imply that with probability at least $1 - \frac{\delta}{2}$ over $S \sim \cD^n$, for all $h \in \cH$, $|\ell_{V^r}(h, S) - \ell_{V^r}(h, \cD)| \leq \frac{\epsilon}{3}.$ This implies, 
\begin{equation}\label{eqn:first_uniform}
\ell_{V^r}(h^*, \cD)  \geq \ell_{V^r}(h^*, S) - \frac{\epsilon}{3}.
\end{equation}

Then, using the fact that $\ell_{U_x^{r-\alpha}} \geq \ell_{V_x^r}$ (Lemma \ref{lem:v_construct}) along with the definition of $h^{r-\alpha}$, we have
\begin{equation}\label{eqn:getting_to_the_bridge}
\ell_{V^r}(h^*, S) - \frac{\epsilon}{3} \geq \ell_{U^{r-\alpha}}(h^*, S) - \frac{\epsilon}{3} \geq \ell_{U^{r-\alpha}}(h^{r-\alpha}, S) - \frac{\epsilon}{3}.
\end{equation}
Applying Lemma \ref{lem:r_works}, we have with probability at least $1- \frac{\delta}{2}$ over $r \sim [\alpha, \gamma]$,
\begin{equation}\label{eqn:the_bridge}
\ell_{U^{r-\alpha}}(h^{r-\alpha}, S) - \frac{\epsilon}{3} \geq \ell_{U^r}(h^r, S) - \frac{2\epsilon}{3}.
\end{equation}
Using $V_x^r \subset U_x^r$ and uniform convergence over $V_x$ one more time, we get that 
\begin{equation}\label{eqn:second_uniform}
\ell_{U^r}(h^r, S) - \frac{2\epsilon}{3} \geq \ell_{V^r}(h^r, S) - \frac{2\epsilon}{3} \geq \ell_{V^r}(h^r, \mathcal{D}) - \epsilon.
\end{equation}
Finally, using Lemma \ref{lem:v_construct} along with the fact that $U_x \subset U_x^{r-\alpha}$, we have
\begin{equation}\label{eqn:lets_get_it}
\ell_{V^r}(h^r, \mathcal{D}) - \epsilon \geq \ell_{U^{r-\alpha}}(h^r, \mathcal{D}) - \epsilon \geq \ell_U(h^r, \mathcal{D}) - \epsilon.
\end{equation}
Combining Equations \ref{eqn:getting_to_V}, \ref{eqn:first_uniform}, \ref{eqn:getting_to_the_bridge}, \ref{eqn:the_bridge}, \ref{eqn:second_uniform}, and \ref{eqn:lets_get_it} with the transitive property, completes the proof, as a simple union bound shows that they all hold simultaneously with probability at least $1 - \delta$, as desired.

\end{proof}

\section*{Acknowledgments}

Robi Bhattacharjee thanks NSF under CNS 1804829 for research support.

%% file: 9.1_appendix.tex
%!TEX root = main.tex

\section{Details for the proof of Theorem \ref{thm:lower_bound}}\label{sec:lower_bound_proof}

\begin{proof}
We want to show that for any $m \in \mathbb{N}$, any learner on $m$ samples must fail with constant probability. Toward this end, set $M = \binom{3m}{m}$, and let $Z^{(M)}_1, Z^{(M)}_2, \dots, Z^{(M)}_M$ be subsets of $\R^d$ as described by Lemma \ref{lem:finding_shatter} (we will drop the superscript in what follows). Let $\mathcal{M}$ denote the set of all subsets of $\{1, \dots, 3m\}$ with exactly $m$ elements. Associate with each $Z_i$ a unique element of $\mathcal{M}$, thus allowing us to rename our subsets as $\{Z_T: T \in \mathcal{M}\}.$ We will now construct a set of robustness regions $U$ from these subsets. For $1 \leq i \leq 3m$, define $$U_{x_i} = \cup_{T: i \in T} Z_T,$$ where $x_i$ is an arbitrary point inside $U_{x_i}$. Note this is well-defined since the $Z^{(M)}$ are mutually disjoint.

By Lemma \ref{lem:finding_shatter}, it follows that if all $x_i$ are given a label of $-1$, then any classifier $h \in \cH_W$ satisfies that $h(z) = 1$ for some subset $T$ and some $z \in Z_T$. However, this will imply that $h$ lacks robustness on all $x \in \{x_i: i \in T\}$, meaning that there are at least $m$ points among $\{x_1, \dots, x_{3m}\}$ where $h$ has robust loss $1$. Furthermore, the second part of Lemma \ref{lem:finding_shatter} implies that for any $T \in \mathcal{M}$, there exists a classifier $h_T$ for which $h_S$ is $-1$ over all $Z_{T'}$ for $T' \neq T$. This implies that $h_T$ is robust at all $x_i$ \textit{except} for $x_i$ with $i \in S$. 

With these observations, we are now prepared to show that for any learner $L$, there exists a distribution $D$ for which $L$ has large expected robust loss. To do this, we use a standard lower bound technique found in \cite{ml_book} that was adapted to the robust setting in \cite{Omar19}. 

 The idea will be to pick $D$ to be the uniform distribution over a random subset of $2m$ points in $\{x_1, \dots, x_{3m}\}$. We will then argue that because $L$ only has access to $m$ points from $D$, it won't be able to distinguish which subset $D$ corresponds to, and this will lead to a large expected loss.

To this end, for any $T \in \cM$, let $D_T$ be the data distribution over $\R^d \times \{\pm 1\}$ where $x$ is chosen at uniform from $\{x_i: i \notin T\}$ and $y$ is always $-1$. We may assume without loss of generality that our learning algorithm, $L$, always outputs a classifier among the set $\{h_T: T \in \cM\}$. This is because Lemma \ref{lem:finding_shatter} implies that any classifier in $h \in \cH_W$ has robust loss that is at least as bad some $h_T$ (namely, if the decision boundary of $h$ crosses $Z_T$). 

Next, let $T, T' \in \cM$ be arbitrary. By definition, $h_T$ lacks robustness on all $x_i$ with $i \in T$, and is perfectly accurate and robust at all other points. It follows that among the $2m$ points in the support of $D_{T'}$, there are $m - |T \cap T'|$ where $h_T$ lacks robustness, implying the the loss of classifier $h_T$ with respect to distribution $D_{T'}$ is $\frac{1}{2} - \frac{|T \cap T'|}{2m}$. Note that this implies that $h_T$ has $0$ robust loss over $D_T$ (thus meeting the first stipulation of Theorem \ref{thm:lower_bound}). 

Finally, we bound the expected loss of the learner $L$ with respect to a uniformly random choice of $D_T$. Let $\cM$ also denote the uniform distribution over itself, and let $\cU$ denote the uniform distribution over $\{1, 2, 3, \dots, 3m\}$. Taking expectations over $T \sim \cM$ and $S \sim D_T^m$, and letting $h_{L(S)}$ denote the classifier learned by $L$, we have that 
\begin{equation*}
\begin{split}
\E_{T \sim \cM} \E_{S \sim D_T^m} \left[\ell_U(h_{L(S)}, D_T)\right] &= \E_{S \sim \cU^m} \E_{T \sim (\cM | S)} \left[\ell_U(h_{L(S)}, D_T)\right] \\
&= \E_{S \sim \cU^m} \E_{T \sim \{T': T' \in \cM, S \cap T' = \emptyset\}}\left[ \frac{1}{2} - \frac{|T \cap L(S)|}{2m} \right].\\
\end{split}
\end{equation*}
To bound the inner expectation, observe that since $|S| = m$, $T'$ has a conditional distribution that is an arbitrary (at uniform) subset of at least $2m$ indices. Since $L(S)$ is fixed, it follows that the probability that any element in $L(S)$ is an element of $T'$ is at most $\frac{1}{2}$, meaning that the expected value of $|T \cap L(S)|$ is at most $\frac{|L(S)|}{2} = \frac{m}{2}$. Substituting this, we have that 
\begin{equation*}
\begin{split}
\E_{T \sim \cM} \E_{S \sim D_T^m} \left[\ell_U(h_{L(S)}, D_T)\right] \geq \E_{S \sim \cU^m} \E_{T \sim \{T': T' \in \cM, S \cap T' = \emptyset\}}\left[ \frac{1}{2} - \frac{m}{4m} \right] = \frac{1}{4}.\\
\end{split}
\end{equation*}
By Markov's inequality, any random variable between $0$ and $1$ with expectation $\frac{1}{4}$ is strictly larger than $\frac{1}{8}$ with probability at least $\frac{1}{7}$. Since the loss above is bounded between $0$ and $1$, it follows that $\Pr_{T \sim \cM}\Pr_{S \sim D_T} [\ell_U(h_{L(S)}, D) > \frac{1}{8}] \geq \frac{1}{7}$. Thus, for some $D = D_T$, the desired claim holds, finish the proof. 
\end{proof}

\section{Sample Oracle Lower Bounds}\label{app: lower bound}
\begin{figure*}[h!]
\centering
	%\begin{subfigure}[b]{0.11\textwidth}
	 %  \centering
		\includegraphics[width=.6\linewidth]{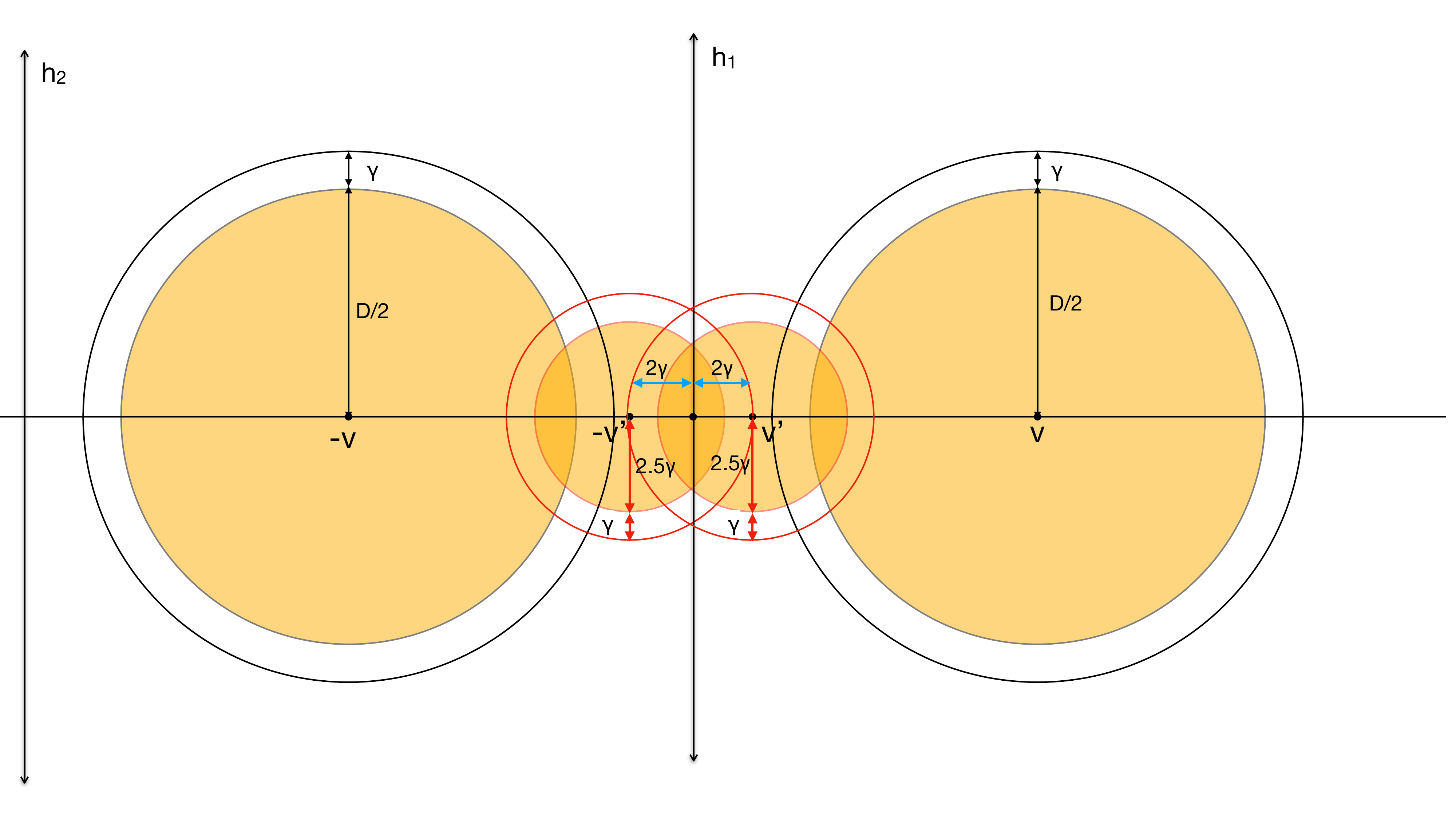}
	%	\caption{The generic task}
		\label{fig:LB}
	%\end{subfigure}
   \caption{Illustration for the sampling oracle lower bound in \propref{prop: lowerbound} in $\R^2$.}
	\label{fig:illustration}
\end{figure*}

We now show a lower bound on the number of oracle calls required for tolerant learning in \citet{Urner22}'s sample oracle model. We first recall the model itself for completeness, focusing on the case of $(\mathbb{R}^d,\ell_2)$ endowed with the standard Lebesgue measure for simplicity.
\begin{defn}[Sampling Oracle {\cite{Urner22}}]
Let $U: \mathbb{R}^d \to P(\mathbb{R}^d)$ be any perturbation function such that $U(x)$ has finite Lebesgue measure for all $x \in \mathbb{R}^d$. The sampling oracle $\mathcal{O}_U$ inputs any $x \in \mathbb{R}^d$, and outputs a sample $y$ from the induced distribution on $U(x)$ under the Lebesgue measure.
\end{defn}
We prove that tolerant learning requires exponentially many calls to the sampling oracle.
\begin{customprop}{6}
For any $D>10\gamma > 0$, there exists a hypothesis class $\mathcal{H}$ and a set of robustness regions, $U$ such that the following holds. There exist constants $\epsilon$ and $\delta$, along with a data distribution $\cD$, such that for any $n > 0$, any learner $L$ that achieves $$\ell_U(L(S), \cD) \leq \min_{h \in \cH} \ell_{U^\gamma}(h, \cD) + \epsilon$$ with probability at least $1-\delta$ over $S \sim \cD^n$ must make at least $\Omega\left(\left(\frac{D}{\gamma}\right)^d\right)$ calls to the sampling oracle. 
\end{customprop}
\begin{proof}[Proof of Proposition \ref{prop: lowerbound}]

Appealing to Yao's Minimax Principle, it is enough to find a class $\mathcal{H}$ and strategy for the adversary (over valid choices of perturbation sets and data distributions) such that any deterministic learner using at most $O((\frac{D}{\gamma})^d)$ oracle calls incurs at least constant error ($\epsilon$) over the optimum in $\mathcal{H}$ with constant probability ($\delta$). 

With this in mind, fix $D_0 =D-9\gamma$, let $r = 4\gamma$, and let $e_1$ denote the first canonical basis vector in $\R^d$. Our (marginal) data distribution will consist of two points in $\R^d$ $\curlybracket{(\frac{D_0}{2}+4\gamma)e_1,-(\frac{D_0}{2}+4\gamma)e_1}$. For the ease of notation, we denote $v \coloneqq (\frac{D_0}{2}+4\gamma)e_1$. Note, $||v- (-v)||_2 = D_0 + 2r$. We now define the underlying hypothesis class $\cH$ which consists of two linear classifiers $\cH \coloneqq \curlybracket{h_{1},h_{2}}$ such that $h_1 = sgn(\langle e_1, \cdot \rangle)$ and $h_2 = (\langle e_1, \cdot \rangle - D_0 - 4\gamma)$. Note that $h_1$ is a perpendicular bisector of the line segment joining $v$ and $-v$, and $h_2$ is parallel to $h_1$ but biased to the left of $v$.

Finally, we construct two perturbation sets with bounded diameter $U$ and $V$. Fix $v' = 2\gamma e_1$.
%points $v'$ and $-v'$ such that $v,v'$ are co-linear, $||x_1 - x_1'|| = \frac{D_0}{2} + 2\gamma$, and $||x_2 - x_2'|| = \frac{D_0}{2} + 2\gamma$. 
We define balls of radius $r>0$ for any given $x \in \R^d$ as $B_2(x, r)\coloneqq$ $\curlybracket{x' \in \R^d: ||x' - x||_2 \le r}$. First, we define a perturbation $U$ and its $\gamma$-perturbed region $U^\gamma$ as follows:
%Let the perturbation sets $U,V$ and their $\gamma$-perturbed sets for $\curlybracket{x_1,-v}$ be defined as 
\begin{gather*}
 U \coloneqq \curlybracket{U_{v},U_{-v}} \textnormal{where for any } x \in \curlybracket{v,-v}, U_{x} = B_2\paren{x, \frac{D_0}{2}},\\
 U^{\gamma}\coloneqq \curlybracket{U_{v}^\gamma,U_{-v}^\gamma} \textnormal{where for any } x \in \curlybracket{v,-v}, U_{x}^{\gamma} = B_2\paren{x, \frac{D_0}{2} +\gamma}
\end{gather*}
Similarly, we define another perturbation set $V$ and its $\gamma$-perturbed region $V^{\gamma}$:
\begin{gather*}
V\coloneqq \curlybracket{V_{v},V_{-v}} \textnormal{where for any } x \in \curlybracket{v,-v}, V_{x} = U_{x}\cup B_2\paren{x',\frac{5\gamma}{2}},\\
V^{\gamma}\coloneqq \curlybracket{V_{v}^\gamma,V_{-v}^\gamma} \textnormal{where for any } x \in \curlybracket{v,-v}, V_{x}^\gamma = U_{x}^\gamma \cup B_2\paren{x',\frac{7\gamma}{2}}
    %V \coloneqq \curlybracket{B\paren{x_1, \frac{D}{2}}\bigcup B\paren{x_1',\frac{5\gamma}{2}},B\paren{-v, \frac{D}{2}}\bigcup B\paren{x_2',\frac{5\gamma}{2}}}\\
 %V^{\gamma} \coloneqq \curlybracket{B\paren{x_1, \frac{D}{2}+\gamma}\bigcup B\paren{x_1',\frac{7\gamma}{2}},B\paren{x_2, \frac{D}{2}}\bigcup B\paren{x_2',\frac{7\gamma}{2}}}
\end{gather*}
where $x' = v'$ or $-v'$ if $x = v$ or $-v$ respectively. We assume that the perturbation set for $\R^d\setminus \curlybracket{v,-v}$ is null for simplicity. 
Observe that $\bigcap\limits_{x' \in \curlybracket{v,-v}} U_{x'} = \emptyset$ and so is the intersection of perturbations in $U^\gamma$.
%First, we note that $V$ is a $\gamma$-\textbf{perturbed regions} of $U$.
But, we note that $\bigcap\limits_{x' \in \curlybracket{v,-v}} V_{x'} \not= \emptyset$. This entire construction is illustrated in Figure \ref{fig:illustration}.
%Furthermore, we can pick $p_0$ such that the measure of the intersection of the sets $V^{\gamma}_{x_1}$ and $V^{\gamma}_{x_2}$ is upper bounded by $\paren{\frac{\gamma}{D}}^d$, i.e. $\mu\paren{V^{\gamma}_{x_1}\cap V^{\gamma}_{x_2}} \le \paren{\frac{\gamma}{D}}^d$.
% As considered in \citet{Urner22}, for a perturbation set $Z = U$ or $Z = V$, a sampling oracle samples based on the induced probability measure $P_{Z_x}^{\mu}$ over $Z_x \in Z$ as follows. For any set $Z'\subseteq Z_x$ in the $\sigma$-algebra over $Z_x$, we define 
% $P_{Z_x}^{\mu}(Z') = \frac{\mu(Z')}{\mu(Z_x)}$.

We are now ready to describe the adversary's strategy, who chooses one of $U$ or $V$ independently with probability $1/2$, and employs a single fixed choice of data distribution $\cD$ where $\Pr[Y = -1 | -v] = \Pr[Y = 1 | v] = 1$, and the marginal distribution is uniform over $v$ and $-v$. Note that if the perturbation set is $U$, then $h_1$ is optimal as $\ell_U\paren{h_1,\cD} = 0$ whereas $\ell_U\paren{h_2,\cD} = 1/2$. On the other hand if $V$ is chosen then $h_2$ is optimal as $\ell_V\paren{h_2,\cD} = \frac{1}{2}$ and $\ell_V\paren{h_1,\cD} = 1$. The idea is to show that the learner cannot distinguish between $U$ and $V$ with high probability, and thus cannot choose the right hypothesis. We note that since the data distribution is fixed and known to the learner, we only need to consider randomness over the sample oracle---labeled samples have no effect on the bound.

More formally, we split our analysis into two cases based on whether or not the learner draws an (oracle) sample in $V^\gamma \setminus U^\gamma$. First, note that conditioned on the fact that the learner draws no such sample, by construction the posterior probability of $U$ is strictly higher than that of $V$. This means the learner's expected excess error is minimized by always outputting $h_1$ on such samples. On the other hand, if the learner observes a sample in $V^\gamma \setminus U^\gamma$, they can always achieve optimal error by outputting $h_2$. 

Since the above learning rule minimizes the learner's expected excess error, it is enough to bound the expected error of this rule:
\begin{align*}
    \mathbb{E}_{Z,S \sim \mathcal{O}_Z}[OPT_Z - \ell_{Z}(\mathcal{A}(S),\cD)] &\geq \frac{1}{2}\Pr[S \subset U^\gamma \wedge Z=V]\\
    &= \frac{1}{2}\Pr[Z=V]\Pr[S \subset U^\gamma | Z=V]\\
    &= \frac{1}{4}\Pr[S \subset U^\gamma | Z=V]
\end{align*}

The key observation is then simply to notice that $\Pr[S \subset U^\gamma | Z=V]$ is constant whenever the learner draws at most $O((\frac{D}{\gamma})^d)$ oracle samples. This follows from the fact that under the induced distribution $P_{V^\gamma}$ on $V$:
\begin{equation}
 P_{V^\gamma}(V^\gamma\setminus U^\gamma) = \frac{\mu(V^\gamma\setminus U^\gamma)}{\mu(V^{\gamma})} \le \frac{\mu(B_2(v',\frac{7\gamma}{2}))}{\mu(U^\gamma_v) + \mu(B_2(v',\frac{7\gamma}{2}))} \le \frac{(\frac{7}{2}\gamma)^d}{D_0^d}\nonumber
\end{equation}
where $\mu$ is the standard Lebesgue measure. Similarly we then have
\begin{align*}
    P_{V^\gamma}(U^\gamma) \ge 1 - \frac{(\frac{7}{2}\gamma)^d}{D_0^d} \label{eq: U}
\end{align*}
and finally that
\begin{align*}
    \Pr[S \subset U^\gamma | Z=V] \geq \paren{1 - \frac{(\frac{7}{2}\gamma)^d}{D_0^d}}^{|S|}
\end{align*}
which is at least some constant when $|S| \leq c(\frac{D_0}{\gamma})^d$ for some sufficiently small absolute constant $c<0$. Since $D_0 = D - 9\gamma$, there exists $c'$ such that this holds when $|S| \leq c'(\frac{D}{\gamma})^d$ which implies the proposition.

\end{proof}
We note that in \citet{Urner22}, the sampling oracle is defined more generally for any \textit{doubling-measure} $\mu$, that is any measure for which there exists some ``doubling-constant'' $C>0$ such that for all $x \in \mathbb{R}^d$ and $r \in \mathbb{R}^+$:
\[
0 < \mu(B(x,2r)) \leq C\mu(B(x,r)) < \infty.
\]
In this more general setting, one can prove a lower bound that scales with the doubling-constant (typically exponential in the associated doubling-dimension of the metric space) simply by appropriately increasing the concentration of measure on $U^\gamma$.

\section{Robust VC for $k$ points}\label{app: robust vc bound}
In this section, we prove that the size-$k$ perturbation sets only cost a $\log(k)$ factor over the VC dimension of the original class. To formalize this, we first recall a few basic definitions standard to the (adversarially robust) learning literature.

\begin{defn}[Robust Loss Class]
Given a hypothesis $h: \cX \to \{0,1\}$ and perturbation function $U:\cX \to P(\cX)$, let $h^\ell_U: X \times \{0,1\}$ be the function over labeled samples measuring the robust loss of $h$:
\[
h^\ell_U(x,y) = \begin{cases}
0 & \text{ if } \ \forall x' \in U(x): h(x')=y\\
1 & \text{ else.}
\end{cases}
\]
The robust loss class of $(\cX,\cH)$ is the hypothesis class over $\cX \times \{0,1\}$ given by:
\[
\mathcal{L}^U_{\cH} \coloneqq \{ h^\ell_U : h \in \cH\}.
\]
\end{defn}
We are interested in analyzing a standard complexity measure of the robust loss class called VC dimension
\begin{defn}[VC Dimension]
The VC dimension of a hypothesis class $(\cX,\cH)$ is the size of largest subset $S \subseteq \cX$ such that $\cH$ obtains all $2^{|S|}$ labelings on $S$. We say such a set is \textit{shattered} by $\cH$.
\end{defn}
% We define the robust VC dimension of a class $(X,H)$ as the VC dimension of its robust loss class.
% \begin{defn}[Robust VC Dimension]
% The robust VC dimension of a hypothesis class $(X,H)$ with respect to perturbation function $U:X \to P(X)$ is $vc(\mathcal{L}^U_H)$.
% \end{defn}
We show the VC dimension of the robust loss class incurs at most $\log(k)$ blow-up over the original class.
\begin{prop}[Overhead of Robust VC]\label{prop:finite-RVC}
Let $(\cX,\cH)$ be a hypothesis class of VC-dimension $d$ and $U: \cX \to P(\cX)$ any perturbation function with support bounded by some $k \in \mathbb{N}$. Then the VC dimension of $\mathcal{L}^U_{\cH}$ is at most $O(d\log(dk))$.
\end{prop}
This result was also independently communicated to us by Omar Montasser. The proof of Proposition \ref{prop:finite-RVC} relies on the classical Sauer-Shelah-Perles lemma, which we recall here for completeness.
\begin{lem}[Sauer-Shelah-Perles \cite{sauer1972density,shelah1972combinatorial}]
Let $(\cX,\cH)$ be a hypothesis class of VC-dimension $d$. Then for any finite subset $S \subseteq \cX$, $\cH$ obtains at most $O(|S|^d)$ distinct labelings on $S$.
\end{lem}
Proposition \ref{prop:finite-RVC} simply follows from using Sauer-Shelah-Perles to bound the total number of permissible patterns across perturbation sets of a sample in the loss space.
\begin{proof}[Proof of Proposition \ref{prop:finite-RVC}] Let $m \in \mathbb{N}$ and assume there exists a sample $S=(x_1,y_1), \ldots, (x_m,y_m)$ that is shattered by $\mathcal{L}^U_{\cH}$. We will show $m \leq O(d\log(kd))$. With this in mind, let $T=\bigcup_{i=1}^m U(x_i)$ denote the set of at most $km$ points corresponding to the robustness regions of our sample. The key observation is the following (essentially trivial) claim:
\begin{claim}
Any two $g_{U}^\ell,h_{U}^\ell \in \mathcal{L}^U_{\cH}$ that give distinct labelings of $S$ correspond to $g,h \in \cH$ with distinct labelings of $T$.
\end{claim} 
By robust shattering, there exist $2^m$ distinct labelings of $S$ by $\mathcal{L}^U_{\cH}$, so the above claim implies $\cH$ must have $2^m$ distinct labelings of $T$. However the latter has at most $O((km)^d)$ labelings by VC dimension, so
\[
2^m \leq O((km)^d) \Rightarrow m \leq O(d\log(dk))
\]
by standard manipulations. Finally, we note the claim is immediate from definition, since the behavior of a function $h_U^\ell \in \mathcal{L}^U_{\cH}$ on $S$ depends only on the labels of its corresponding hypothesis $h \in \cH$ on $T$ by definition. 
\end{proof}

\section{Proof of Theorem \ref{thm:upper_bound_general}}\label{sec:proof_extension}

We begin by defining $v_{ball}$, which is the adversarial VC dimension when the robustness regions are all balls of a fixed radius. We start by precisely defining these robustness regions. 

\begin{defn}
Let $U^r$ be the set of robustness regions defined by $\{U_x^r = B(x, r)\}$, where $B(x, r)$ denotes the closed ball of $\ell_2$-radius $r$ centered at $x$. 
\end{defn}

We now define the adversarial VC dimension of a set of classifiers $\cH$ for a fixed set of regions, $U^r$.

\begin{defn}
Let $\cH$ be a set of classifier. Then the adversarial VC dimension of $\cH$ with respect to $U^r$ is the maximum integer $v$, for which there exist $v$ labeled points, $(x_1, y_1), \dots, (x_r, y_r)$ so that for any subset $S \subset \{(x_1, y_1) ,\dots, (x_r, y_r)\}$, there exists $h_S \in \cH$ with $$\ell(h_S, (x_i, y_i) = \begin{cases}0 & i \in S \\ 1 & i \notin S\end{cases}.$$ We denote this by $v_{ball}^r$.
\end{defn}

Finally, we define $v_{ball}$ as the maximum value of $v_{ball}^r$ over all $r > 0$. Note that this quantity has been well studied -- for example \cite{Cullina18} shows that for linear classifiers, $v_{ball} = O(d)$. 

\paragraph{Proving Theorem \ref{thm:upper_bound_general}} We now turn our attention towards the proof. The key observation is that the main steps from the proof of Theorem \ref{thm:upper_bound} \textit{perfectly} carry over. In particular, Lemma \ref{lem:r_works} exactly holds in this setting, and the argument given in the proof of Theorem \ref{thm:upper_bound} also holds provided that an appropriate choice of $V$ exists. The only issue arises from Lemma \ref{lem:v_construct}, which requires that $\cH$ be regular. To remedy this, we now state and prove a different version of this lemma that uses a union of balls (of fixed radius) for $V_x$ rather than a finite set of points. 

\begin{lem}\label{lem:v_construct_new}
Let $\cH$ by an arbitrary hypothesis class. For all $r \in [\frac{\epsilon\delta\gamma}{7}, \gamma]$, let $\alpha$, $U^r$ and $U^{r-\alpha}$ be as described in the proof of Theorem \ref{thm:upper_bound}. Then there exists a set of robustness regions $V^r = \{V_x^r: x \in \reals^d\}$ satisfying the following two properties.  
\begin{enumerate}
	\item $V_x^r$ is a union of $O\left(\left(\frac{D}{\epsilon\delta\gamma}\right)^d\right)$ balls of radius , where $D$ denotes the maximum diameter of $U_x$. 
	\item Let $\alpha = \frac{\epsilon\delta\gamma}{7}$. For all labeled points $(x,y)$ and for all classifiers $h \in \cH$, $$\ell_{U^{r - \alpha}}(h, (x,y)) \leq \ell_{V^r}(h, (x,y)) \leq \ell_{U^r}(h, (x,y)).$$
\end{enumerate}
\end{lem}

\begin{proof}
Since $U_x^{r-\alpha}$ has diameter at most $D$, it follows that it can be covered with $O\left(\left(\frac{D}{\epsilon\delta\gamma}\right)^d\right)$ balls of radius $\frac{\alpha}{2}$. We let $V_x^r$ be any such cover that is minimal (meaning that (1.) is satisfied), meaning that each ball intersects $U_x^{r-\alpha}$. It follows that for all $x$, $U_x^{r-\alpha} \subseteq V_x^r \subseteq U_x^r$, which immediately implies (2.) and completes the proof. 
\end{proof}

Finally, to prove Theorem \ref{thm:upper_bound_general}, we note that the proof of Theorem \ref{thm:upper_bound} essentially works. The only differences are that instead of bounding the robust VC dimension of $\cH$ with respect to $V_X$ in terms of $v$, we must use $v_{ball}$ as we are now considering unions of balls rather than points. As a detail, note that we are using the following minor modification of Proposition \ref{prop:finite-RVC} to bound the robust VC dimension of unions of balls using the robust VC dimension for balls. 
\begin{prop}
Let $(\cX,\cH)$ be a hypothesis class whose robust loss class with respect to $r$-balls has VC dimension $v_{\text{ball}}^r$. Then the loss class of $(\cX,\cH)$ with respect to perturbations that are a union of at most $k$ $r$-balls has VC dimension at most $O(v_{\text{ball}}^r\log(v_{\text{ball}}^rk))$.
\end{prop}
\begin{proof}
The proof is largely the same as \ref{prop:finite-RVC}. Denote the original perturbation family as $U$, and the $k$-union perturbation family by $U^k$. Given a sample $S=(x_1,y_1),\ldots,(x_m,y_m)$, let $C_i$ denote the centers of the at most $k$ balls appearing in the perturbation set of $x_i$. It is enough to observe that any two distinct labelings of $S=(x_1,y_1),\ldots,(x_m,y_m)$ by $\mathcal{L}^{U^k}_{\cH}$ correspond to distinct labelings of the extended sample $T=\bigcup_{i=1}^m (C_i,y_i)$ with respect to $\mathcal{L}^{U}_{\cH}$, where $(C_i,y_i)$ denotes the sample $\bigcup_{c \in C_i} (c,y_1)$. The bound then follows from the same double counting argument as in Proposition \ref{prop:finite-RVC}.
\end{proof}